%% file: arxiv_submission_013120.tex
\documentclass{article}
\usepackage{ijcai20}

\usepackage{times}
\usepackage{soul}
\usepackage{url}
\usepackage[hidelinks]{hyperref}
\usepackage[utf8]{inputenc}
\usepackage[small,belowskip=-8pt,aboveskip=2pt]{caption}
\usepackage{graphicx}
\usepackage{amsmath}
\usepackage{amsthm}
\usepackage{booktabs}
\usepackage{enumitem}

\usepackage[ruled,noend]{algorithm2e}

\newtheorem{theorem}{Theorem}
\newtheorem{lemma}{Lemma}
\newcommand{\eg}{{\em e.g.}}

\usepackage{xcolor}

\newcommand\threecolor{\textsc{3-Color}}

\newcounter{int}
\makeatletter
\newcommand{\citen}[1] {\setcounter{int}{0}\@for\tmp:=#1\do{%
\ifnum \value{int}>0; \fi%
\setcounter{int}{1}%
\citeauthor{\tmp} \shortcite{\tmp}}}
\makeatother
\makeatletter
\title{An Algorithm for Multi-Attribute Diverse Matching}

\author{
 Saba Ahmadi$^1$,
 Faez Ahmed$^2$,
 John P.\ Dickerson$^1$,
 Mark Fuge$^3$,
 Samir Khuller$^4$
 \\ 
 $^1$ Department of Computer Science, University of Maryland\\
 $^2$ Department of Mechanical Engineering, Northwestern University\\
 $^3$ Department of Mechanical Engineering, University of Maryland\\
 $^4$ Department of Computer Science, Northwestern University\\
 saba@cs.umd.edu,
 faez@northwestern.edu,
 john@cs.umd.edu,
 fuge@umd.edu,
 samir.khuller@northwestern.edu
}

\begin{document}

\maketitle

\begin{abstract}
\input{abstract}
\end{abstract}

\input{introduction}
\input{related}
\input{prelims}
\input{algorithm}
\input{analysis}
\input{general_weights}
\input{experiments}

\input{conclusion}
{\footnotesize
\bibliographystyle{named}
\bibliography{ijcai18,new_refs}%
}
\end{document}

%% file: abstract.tex
Bipartite $b$-matching, where agents on one side of a market are matched to one or more agents or items on the other, is a classical model that is used in myriad application areas such as healthcare, advertising, education, and general resource allocation. Traditionally, the primary goal of such models is to maximize a linear function of the constituent matches (e.g., linear social welfare maximization) subject to some constraints. Recent work has studied a new goal of balancing whole-match \emph{diversity} and economic efficiency, where the objective is instead a monotone submodular function over the matching.  Basic versions of this problem are solvable in polynomial time. In this work, we prove that the problem of simultaneously maximizing diversity along several features (e.g., country of citizenship, gender, skills) is NP-hard. To address this problem, we develop the first combinatorial algorithm that constructs provably-optimal diverse $b$-matchings in pseudo-polynomial time. We also provide a Mixed-Integer Quadratic formulation for the same problem and show that our method guarantees optimal solutions and takes less computation time for a reviewer assignment application. 


%% file: introduction.tex
\section{Introduction}
The bipartite matching problem occurs in many applications such as healthcare, advertising, and general resource allocation. Weighted bipartite $b$-matching is a generalization of this problem where each node on one side of the market can be matched to many items from the other side, and where edges may also have associated real-valued weights. Examples of weighted bipartite $b$-matching include assigning children to schools~\cite{drummond2015sat,kurata2017controlled}, reviewers to manuscripts~\cite{charlin2013toronto,Liu:2014:RMP:2645710.2645749}, and donor organs to patients~\cite{Dickerson:2015:FCH:2887007.2887094,Bertsimas19:Balancing}.

\citen{Ahmed:2017:DWB:3171642.3171649} introduced the notion of \textit{diverse} bipartite $b$-matching, where the goal was to simultaneously maximize the ``efficiency'' of an assignment along with its ``diversity.'' For example, a firm might want to hire several highly-skilled workers, but if that firm also cares about diversity it may want to ensure that some of those hires occur across marginalized categories of employees. They proposed an objective which combined economic efficiency and diversity demonstrating that, in practice, reducing the efficiency of a matching by small amounts can often lead to significant gains in diversity across a matching. However, their formulation was limited to diversity for a single feature. It also relied on solving a general Mixed-Integer Quadratic Program (MIQP), which is flexible but computationally intractable.


\begin{figure}[ht!]
    \centering
    \includegraphics[width=0.46\textwidth]{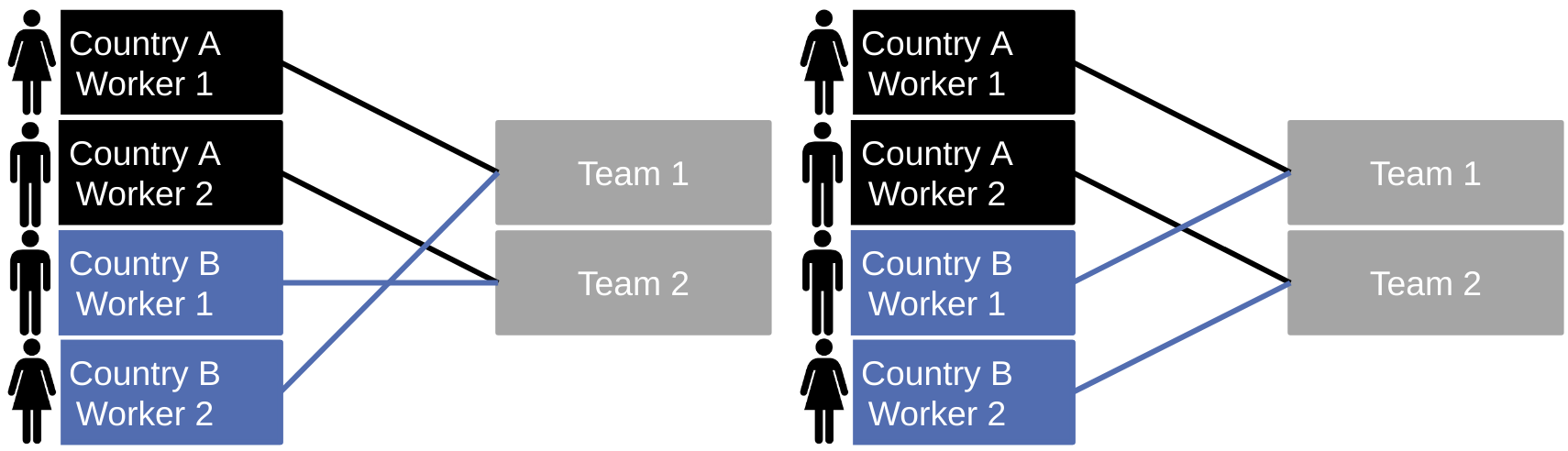}
    \caption{An illustrative example of single feature diverse matching (left) versus multi-feature diverse matching (right); here, the matching creates teams with workers from each country and gender.}
    \label{fig:matrix}
\end{figure}

In this work, we generalize the diverse matching problem and 
introduce matchings where each worker has \emph{multiple} features (e.g., country of origin, gender) and our goal is to form diverse teams with respect to all these features. We found that the problem with a single feature, studied by~\citen{Ahmed:2017:DWB:3171642.3171649}, can be reduced to a minimum quadratic cost maximum flow formulation and solved in polynomial time by an existing algorithm~\cite{Minoux1986}. In contrast, we provide NP-hardness results for the general case of multiple features.

\noindent\textbf{Our contributions.}
The paper's main contributions follow:
\setlist{nolistsep}
\begin{itemize}[noitemsep,leftmargin=*]

    \item We provide the first pseudo-polynomial time algorithm for the diverse bipartite $b$-matching w.r.t.\ multiple features problem with class-specific weights.\footnote{That is, under conditions when the cost of assigning all items from one category to an item on the other side of the graph is the same. This holds when, e.g., one is matching academic papers to reviewers where each reviewer can specify exactly one field of expertise and the cost of assigning a paper to any of the reviewers \emph{within} the same field is the same but differs \textit{across} fields.} The key insight lies in detecting \emph{negative cycles} in an auxiliary graph representation, which we use to either provide incremental improvements to the incumbent diverse matching or prove that our negative-cycle-detection algorithms have found a globally-optimal matching. We also provide a general MIQP formulation for this problem.
        \item We then extend the algorithm to the diverse bipartite $b$-matching problems with \textit{general} edge weights, where edge weights of nodes within a category can be different.
    \item Lastly, we demonstrate our algorithm's applicability to paper-reviewer matching. Our algorithm takes less time to converge to an optimal solution than the proposed MIQP approach (using a state-of-the-art commercial solver). 
    
\end{itemize}

%% file: related.tex
\section{Related Work}
Matching people to form diverse teams leverages the intersection of two past areas of research: the role of team diversity in collaborative work and how diversity among groups of resources is measured and used to form/match teams. Compared to related work, this paper provides a practical, high-performing method to perform diverse $b$-matching that can enable applications like diverse team formation or diverse resource allocation. Below we will use the example of diverse team formation (for example, in project teams within a company) to provide a concrete example to place prior work in context; however, our proposed approach is generally applicable to any diverse matching problem.

In the example of forming teams, the traditional approach is to use weighted bipartite $b$-matching (WBM) methods~\cite{basu2015task}. 
These methods maximize the total weight of the matching while satisfying some constraints. However, there are two major issues with these approaches. First, it assumes that the value provided by a person in a team is always fixed and independent of who else is in the team. This assumption may not hold in many cases. A new team member may provide more added value to the team if she is added to a smaller team compared to the case if she is added to a larger team. This property of diminishing marginal utility can be mathematically captured by a family of functions called submodular functions. Second, existing approaches do not account for diversity \textit{within} a team, where teams with workers from different backgrounds may be desirable. For example, different types of worker diversity have a direct impact on the success rate of tasks~\cite{ross2010crowdworkers}. Likewise, firms with a higher number of employees with higher education and diversity in the types of educations have a higher likelihood of innovating~\cite{ostergaard2011does} and increasing revenue for firms~\cite{hunt2015diversity}. In this paper, we address both these issues.

Past researchers have generally measured diversity by defining some notion of \textit{coverage}\textemdash that is, a diverse set is one that covers the space of available variation. Mathematically, researchers have done so via the use of \textit{submodular functions}, which encode the notion of diminishing returns~\cite{lin2012learning}; that is, as one adds items to a set that are similar to previous items, one gains less utility if the existing items in the set already ``cover'' the characteristics added by that new item. For example, many previous diversity metrics used in the information retrieval or search communities---including Maximum Marginal Relevance (MMR)~\cite{carbonell1998use} and Determinantal Point Processes~\cite{kulesza2012determinantal}---are instances of submodular functions. These functions can model notions of coverage, representation, and diversity~\cite{ahmed2018ranking} and they have been shown to achieve top results on common automatic document summarization benchmarks\textemdash \eg, at the Document Understanding Conference~\cite{lin2012learning}. 

Within matching, our work is closest to that of~\citen{Ahmed:2017:DWB:3171642.3171649}, which used a supermodular function to propose a diverse matching optimization method. Other researchers have also approached similar problems, with diversity either as an objective or as a constraint. For instance,~\citen{golz2018migration} match migrants to localities in a way that maximizes the expected number of migrants who find employment. \citen{benabbou2018diversity} study the trade-off between diversity and social welfare for the Singapore housing allocation. They model the problem as an extension of the classic assignment problem, with additional diversity constraints. \citen{lian2018conference} solve the assignment problem when preferences from one side over the other side are given and both sides have capacity constraints. They use order weighted averages to propose a polynomial-time algorithm which leads to high quality and more fair assignments. \citen{Agrawal18:Proportional} show that a simple iterative proportional allocation algorithm can be tuned to produce maximum matching with high entropy.  
Finally,~\citen{Kobren:2019:PML:3292500.3330899} proposed two fairness-promoting algorithms for the paper-reviewer matching problem. They demonstrate that their algorithm achieves higher utility compared to state of the art matching algorithms that optimize for fairness only. 
In contrast, our goal is to develop an algorithm for finding the optimal assignment which maximizes utility as well as diversity along multiple features as an objective---along with having constraints on workload. 

We define a utility function that can be tuned to balance the diversity and total weight of matching. The diversity function is inspired by the Herfindahl index~\cite{hirschman1964paternity}, which is a statistical  measure of concentration and commonly used in economics. We provide a new algorithm that models the problem using an auxiliary graph and uses a heuristic improvement of the negative cycle detection of Bellman-Ford by~\citen{Goldberg93aheuristic}\footnote{We used the negative cycle detection algorithm by~\citen{Goldberg93aheuristic}. \citen{Cherkassky93shortestpaths} compared the performance of multiple negative cycle detection algorithms, and the algorithm by~\citen{Goldberg93aheuristic} was one of the fastest.} to find negative cycles and cancel them on a new graph to obtain an optimal solution for the original problem.

%% file: prelims.tex
\section{Preliminaries}
\label{sec:prelim}
In this section, we first define the preliminaries for a diverse matching problem, where workers are to be matched to teams and each team wants workers belonging to a diverse set of features.
In our problem, we are given a set of features for the workers. Let $\mathcal{F} = \{f_1,\cdots, f_{|\mathcal{F}|}\}$ denote the feature set for the workers. An example of a feature set could be \{country of citizenship, gender\}. Each feature $f_k\in \mathcal{F}$ has one of the values $\mathcal{F}_k = \{f_{k,1},\cdots,f_{k, |\mathcal{F}_k|}\}$. Let $|\mathcal{F}_{k,k'}|$ denote the number of workers having value $f_{k,k'}$ for feature $f_k$.
The set of workers is denoted by $X=\{x_1,\ldots,x_n\}$. We wish to form a set of teams $\{T_1,\ldots,T_t\}$ of the workers where each team $T_i$ has a demand of $d_i$, specifying the number of workers that need to be assigned to it. Each worker can be assigned to exactly one team. 

The diversity of an assignment is denoted by $D$ and is equal to $\sum_{k=1}^{|\mathcal{F}|} \lambda_k D_k$, where $D_k$ shows the diversity w.r.t.\ feature $f_k$, and $\lambda_k\in \mathcal{Z}^+$ is a constant. Let $c_{i,k,k'}$ denote the number of workers in $T_i$ having value $f_{k,k'}\in \mathcal{F}_k$ for feature $f_k$. Then, $D_k = \sum_{i=1}^t\sum_{k'=1}^{|\mathcal{F}_k|}c_{i,k,k'}^2$.
To facilitate explanation, we assume throughout this paper that the country of origin is the $1$st feature, therefore, the number of workers assigned to team $T_i$ from $j$-th country is denoted by $c_{i,1,j}$. The cost of assigning each worker from $j$-th country to team $T_i$ is denoted by $u_{i,j}\in \mathcal{Z}^+$. We assume all costs are integers. The total cost of an assignment is $TU = \sum_{i=1}^t \sum_{j=1}^{|\mathcal{F}_1|} u_{i,j}\cdot c_{i,1,j}$.

Our goal is to minimize the objective function which is equal to $\lambda \cdot D + \lambda_0 \cdot TU$, where $\lambda\in \mathcal{Z^+}^{|\mathcal{F}|}$, and $\lambda_0 \in \mathcal{Z^+}$ is a constant. Next, we provide Theorem \ref{thm:NP-completeness-multiple-features}, which shows that this problem is NP-hard.
\begin{theorem}
\label{thm:NP-completeness-multiple-features}
Minimizing the supermodular diversity function w.r.t multiple features is NP-hard.
\end{theorem}
\begin{proof}
We show a reduction from the \threecolor{} problem which is as follows: given a graph $G=(V,E)$ with $n$ vertices, does there exist a coloring with $n_1$ vertices of color $c_1$, $n_2$ vertices of color $c_2$, and $n_3$ vertices of color $c_3$, such that no two adjacent vertices receive the same color, and all the vertices are colored?

The reduction is as following: In \threecolor{}, assign a feature $f_k$ to each edge $e_k = (v_{k_1},v_{k_2})\in E$, and a worker to each vertex. Let $f_{k,i}$ denote the value of $f_k$ for the worker corresponding to $v_i\in V$. Let $f_{k,i} = i$ if $i\neq k_1, k_2$. Otherwise, let $f_{k,i}=0$.
The goal is to form three teams $T_1, T_2, T_3$ with demands $d_1 = n_1, d_2 = n_2, d_3 = n_3$, respectively. We assume that all the costs of assigning workers to the teams are zero, therefore the objective function is to minimize the total diversity. Consider an arbitrary edge $e_k = (v_{k_1}, v_{k_2})$. If the endpoints of $e_k$ belong to different teams, $f_k$ contributes $n_1+n_2+n_3$ to the objective function since all the workers inside a team have different values for $f_k$. Otherwise, it contributes $n_1+n_2+n_3-2+2^2$ since workers corresponding to $v_{k_1}, v_{k_2}$ are the only workers having the same value for $f_k$ inside a team.
If the cost of the optimal solution for the diverse matching problem is $(n_1+n_2+n_3)\cdot|E|$, there does not exist a pair of workers in a team where the vertices corresponding to them are neighbouring in $G$. Otherwise if the cost of the optimal solution is more than $(n_1+n_2+n_3)\cdot|E|$, the \threecolor{} instance is infeasible.
\end{proof}





We are interested in solving this NP-hard problem. We begin by presenting two different representations of instances of the problem: one in matrix form (used for expositional ease), and the other in graph form (used to build our optimal diverse matching algorithm in Section~\ref{algorithm}). 

\textbf{Matrix Representation:}


An example of matrix representation with three teams and two countries and two genders is shown in Fig.~\ref{fig:matrix-exchange}. In this representation, each column $V_j$ corresponds to a feature set $v_j = \{v_{j,1}, \cdots, v_{j,|\mathcal{F}|}\}$, where $\forall 1\leq k\leq|\mathcal{F}|, v_{j,k} \in \mathcal{F}_{k}$.
Each row corresponds to a team. Entry $w_{i,j}$ shows the number of workers with feature set $v_j$ assigned to $T_i$. We introduce a \emph{dummy team} $T_0$, and $w_{0,j}$ shows the number of workers with feature set $v_j$ who are not assigned to any team. 

\textbf{Matching Representation:}
In this representation, a bipartite graph $G=(\mathcal{X}\cup \mathcal{T}, E)$ is given. The nodes in $\mathcal{X}$ correspond to the workers, and are partitioned into $|\mathcal{V}|$ subsets $V_1,\cdots, V_{|\mathcal{V}|}$, where each subset corresponds to the feature set for a column in the matrix representation. Each vertex in $\mathcal{T}$ corresponds to one team. The assignment of workers to teams forms a $b$-matching, where the degree of node $T_i\in \mathcal{T}$ is $d_i$, and the degree of node $x\in \mathcal{X}$ is at most one.

\textbf{Local Exchange:} 
A local exchange happens when a group of teams decides to transfer one or more workers between each other while maintaining the total number of workers in each of them.
The exchange is done in a way that the initial demands of all the teams are fulfilled. 
Arrows in Fig.~\ref{fig:matrix-exchange} show a local exchange in a matrix representation. 

In this exchange, one worker from $V_2$ is moved from $T_3$ to $T_1$. Two workers from $V_1$ are moved. One is moved from $T_1$ to $T_2$, and the other one is moved from $T_2$ to $T_3$. The set of edges of local exchange in a matrix representation is called a cycle. The source-transitions of a cycle are the cells without any input edges, and the sink-transitions are the cells without any output edges. In Fig.~\ref{fig:matrix-exchange}, the nodes corresponding to $w_{3,2}$ and $w_{1,1}$ are source-transitions nodes, and the nodes corresponding to $w_{1,2}$ and $w_{3,1}$ are sink-transition nodes.

Figure~\ref{fig:matching-exchange} shows the same local exchange operation using a matching representation. In this figure, the black matching shows the initial assignment, and the dotted red matching shows the assignment after the exchange operation is done.

\begin{figure}
    \centering
    \includegraphics[width=0.30\textwidth]{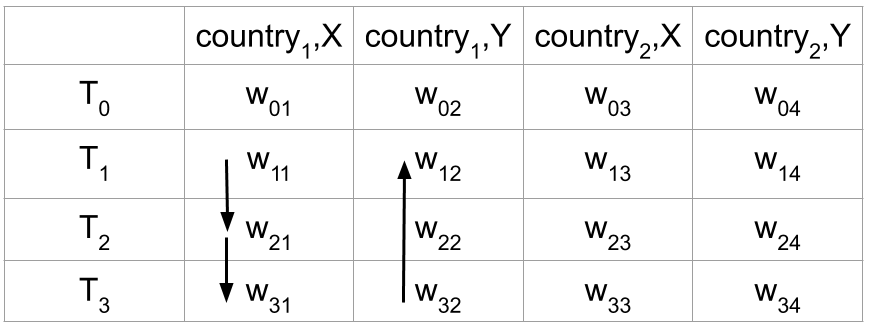}
    \caption{Matrix representation of three teams and workers from two countries and two genders. Dummy team $T_0$ accommodates unassigned workers. Arrows represent a local exchange.}
    \label{fig:matrix-exchange}
\end{figure}

\begin{figure}
    \centering
        \includegraphics[width=0.3\textwidth]{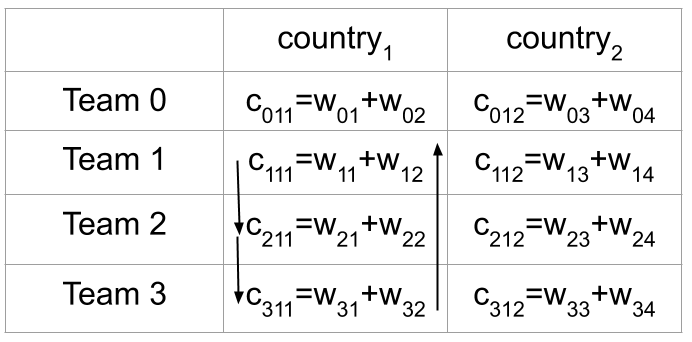}
    \caption{Matrix representation embedding w.r.t country.}
    \label{fig:embedding-country}
\end{figure}

\begin{figure}
    \centering
        \includegraphics[width=0.3\textwidth]{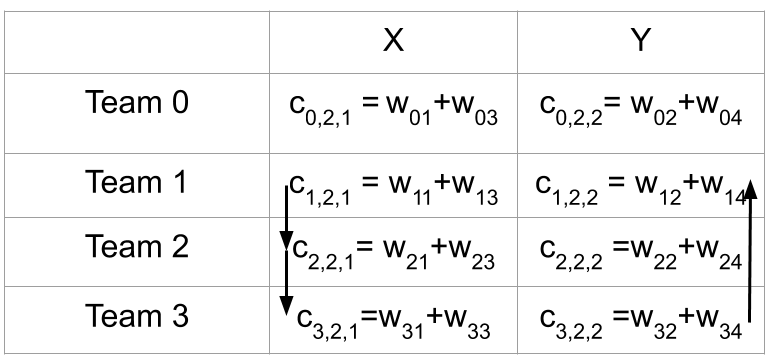}
    \caption{Matrix Representation Embedding w.r.t to gender.}
    \label{fig:embedding-gender}
\end{figure}

\begin{figure}
    \centering
        \includegraphics[width=0.37\textwidth]{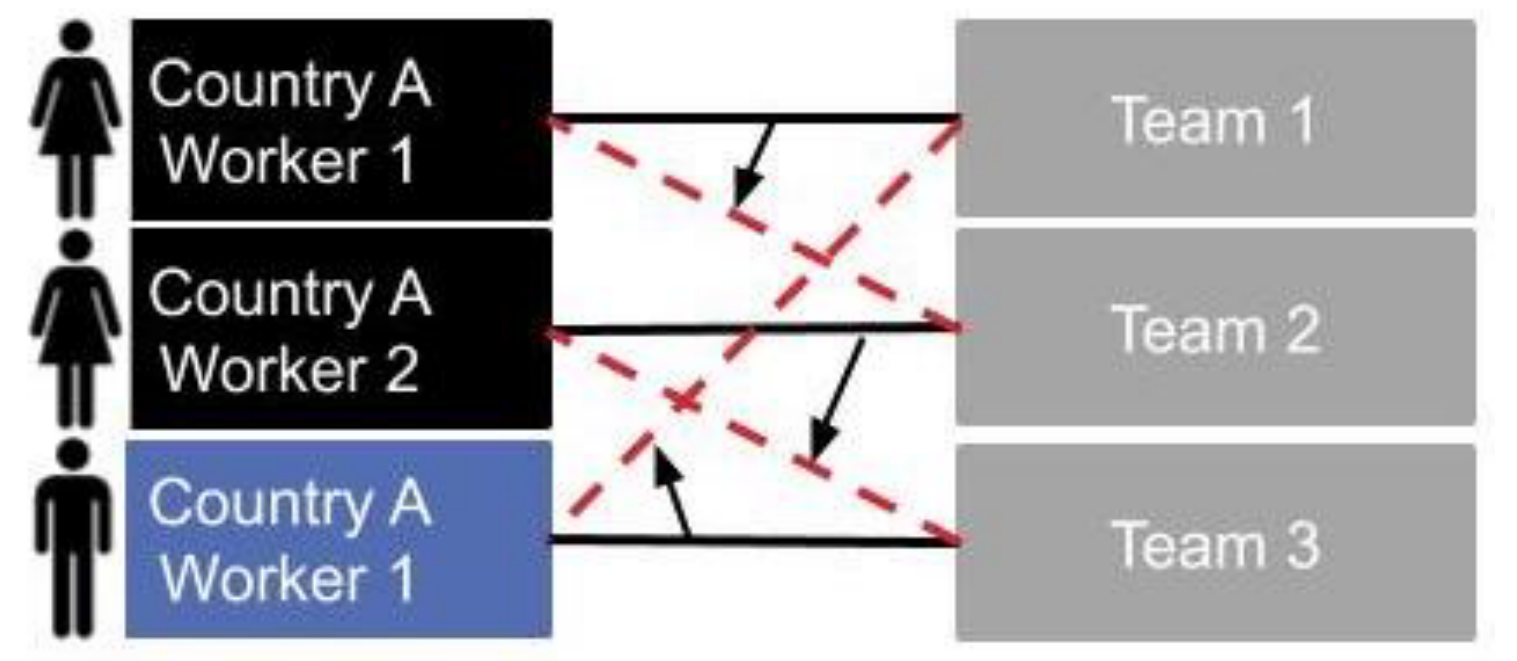}
    \caption{Local exchange operation (in matching representation).}
    \label{fig:matching-exchange}
\end{figure}

\textbf{Gain of a local exchange: } 
Our goal is to minimize the objective function $f$, by doing some local exchanges. To find out, we first calculate the marginal gain from a given exchange operation which is the difference between the objective values before and after a local exchange. 
In order to simplify this concept, we use the following definition:

\textbf{Embedding of Matrix Representation:}
Consider a given matrix representation $M$, it can be embedded into a matrix $M_{k}$ for a fixed feature $f_k$ in the following way: all the columns in $M$ corresponding to the same value $f_{k,k'}$ of $f_k$, are combined into a single column in $M_{k}$. For example, embedding of the matrix representation in Figure~\ref{fig:matrix-exchange} into $M_1, M_2$ w.r.t. the features country and gender are shown in Figures~\ref{fig:embedding-country} and \ref{fig:embedding-gender}.
Since in $M_1$, the number of people assigned from each country to each team is not changed, $\Delta(\lambda_0\cdot TU + \lambda_1 D_1) = 0$. Accoridng to $M_2$, $\Delta(\lambda_2 D_2) = \lambda_2\big((c_{3,2,2}-1)^2-(c_{3,2,2})^2+ (c_{1,2,2}+1)^2-(c_{1,2,2})^2+(c_{1,2,1}-1)^2-c_{1,2,1}^2+ (c_{3,2,1}+1)^2-c_{3,2,1}^2\big)$.

It can be seen that the contribution of the nodes which are not source-transition or sink-transition to the gain of a local exchange is zero (all the nodes in the local exchange in Figure~\ref{fig:embedding-country}, and the node corresponding to $c_{2,2,1}$ in $M_2$). If the net gain, i.e. $\Delta(\lambda_0\cdot TU + \lambda_1 D_1 + \lambda_2 D_2)$, is negative, then the local exchange can be considered beneficial and we can transfer the workers.

%% file: algorithm.tex
\section{Negative-Cycle-Detection-based Algorithms}
\label{algorithm}

In this section, we explain our algorithm for finding the optimum assignment. First, we build an auxiliary graph $G'$. For each team $T_i$, there is a switch in $G'$ with $|\mathcal{V}|$ input ports, and $|\mathcal{V}|$ output ports, where $|\mathcal{V}|$ is the number of columns in the matrix representation. Each port is a node in $G'$, and each switch is a directed bipartite graph, with edges going from its input ports (nodes) to its output ports. In Figure~\ref{fig:graph-exchange}, each box is a switch.
A dummy team $T_0$ is introduced to accommodate all unassigned workers in the matching.
Inside a switch $T_i$, there is a directed edge from each input port to each output port. If the directed edge is connecting two ports such that their corresponding combinations of features do not have the same value for any features, the weight of this edge is equal to zero. Otherwise, per each feature $f_k$ that has the same value, $-2\lambda_k$ is added to the weight of this edge.

The reason behind assigning these weights to the edges is to make sure in a local exchange, considering a fixed feature $f_k$, the nodes which are not a source-transition or a sink-transition w.r.t. $M_{k}$, have zero contribution to $\Delta(D_k)$.

For each pair of teams $T_{i_1}$ and $T_{i_2}$ where ${i_1}\neq {i_2}$, and for each feature combination $v_j$, there is a directed edge from output port $O_{j}^{i_1}$ of switch $T_{i_1}$ to the input port $I_{j}^{i_2}$ of switch $T_{i_2}$, and weight of this edge captures the difference in the objective function when in the matrix representation a person in column $V_j$ (with feature set $v_j$) is moved from $T_{i_1}$ to $T_{i,2}$. 



Each cycle in this graph is corresponding to a cycle in a matrix representation and local exchanges along them have the same gain. Figure~\ref{fig:graph-exchange} shows a cycle which is corresponding to the cycles in Figures~\ref{fig:matrix-exchange} and~\ref{fig:matching-exchange}.


After constructing the auxiliary graph, we run Algorithm~\ref{alg-neg-cycle-detection}. 
Algorithm~\ref{alg-neg-cycle-detection} moves workers from one team to another if it detects a negative cycle. 

\begin{algorithm}
\SetAlgoLined
\SetKwInOut{Input}{Input}
\SetKwInOut{Output}{Output}
\Input{Directed weighted graph $G'$, initial feasible $b$-matching $Q$ which satisfies team demands.}
\Output{Optimal diverse $b$-matching}

 \While{$\exists$ a negative cycle $C\in G'$}{
  // Perform a local exchange operation along $C$\;
  \For{$e\in C$}{
    // Assume edge $e$ is from output port $O_{j}^{i_1}$ of team $T_{i_1}$ to input port $I_{j}^{i_2}$ of another team $T_{i_2}$\;
    // Move one worker with feature set $v_j = \{f_{1,k'_1}, \cdots, f_{|\mathcal{F}|, k'_{|\mathcal{F}|}}\}$ from team $T_{i_1}$ to team $T_{i_2}$:\\
    $\forall k\in \{1,\cdots, |\mathcal{F}|\}$:
    $c_{i_1,k,k'_k}-=1, c_{i_2,k,k'_k}+=1$\;
    Update weight of edges of $G'$ w.r.t to the new values of $c_{i_1,k,k'_k}$, and $c_{i_2,k,k'_k}$\;
  }
 }
 \caption{Find optimal diverse $b$-matching}
  \label{alg-neg-cycle-detection}
\end{algorithm}
Algorithm~\ref{alg-neg-cycle-detection} takes as input an initial feasible solution $Q$ as input. To find $Q$, we first find a feasible solution, which satisfies all the demand constraints. 
In order to find an initial feasible solution, in each iteration, consider the first subset of workers in the the bipartite graph G ($V_j$) with at least one un-assigned worker, and the first team ($T_i$) such that the number of workers assigned to it is less than its demand (In the first iteration, we start with $V_1, T_1$, and all the workers are un-assigned). 
Assign un-assigned workers from $V_j$ to $T_i$, until either demand of $T_i$ is fully satisfied, in this case, move to the next team $(i=i+1)$, or all the workers from $V_j$ are assigned, then let $j=j+1$. Repeat this procedure until all the demand constraints are satisfied. Time complexity of this procedure is $\mathcal{O}(|\mathcal{V}|+t)$.

In Algorithm~\ref{alg-neg-cycle-detection}, any negative cycle detection algorithm can be used to detect negative cycles in $G'$. We use a heuristic improvement of Bellman-Ford proposed by Goldberg and Radzik~\cite{Goldberg93aheuristic} in our experiments.
\begin{figure}
    \centering
    \includegraphics[width=0.35\textwidth, scale = 0.5]{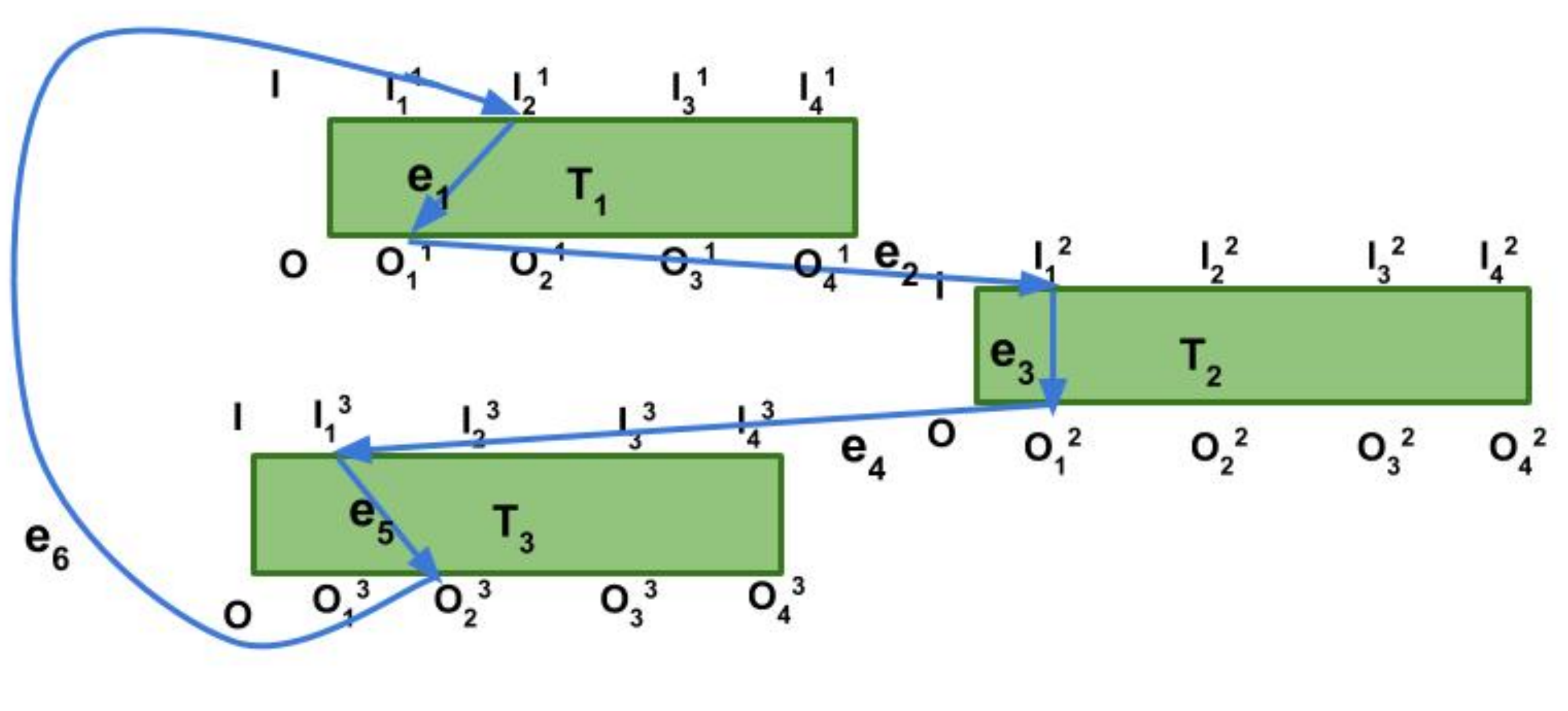}
    \caption{A local Exchange in graph representation. 
    }
    \label{fig:graph-exchange}
\end{figure}

%% file: analysis.tex
\section{Proof of Optimality}
In this section, 
we prove that Algorithm~\ref{alg-neg-cycle-detection} gives the optimum solution for diverse bipartite $b$-matching problem. 

Assume after the algorithm ends, the final assignment is a local optimum $P$, and the optimum solution is $P^*$. 
Consider the matching representations of $P$ and $P^*$. The symmetric difference of $P$ and $P^*$ ($P\oplus P^*$) can be decomposed into a set of alternating cycles and paths of even length. The reason that the length of alternating paths is even is that size of both of the matchings is equal: $|P| = |P^*|=\sum_{i=1}^t d_i$.

Each local exchange along an alternating cycle corresponds to a cycle in the matrix representation. A local exchange along an alternating path corresponds to a cycle in a matrix representation which includes vertices from row $T_0$. 



Before proving Thm.~\ref{proof-correctness}, we need the following definitions:

\textbf{Maximal Cycle:} A cycle $y$ in a matrix representation $M$ is maximal if its source-transitions (nodes with zero incoming edges) and sink-transitions (nodes with zero outgoing edges) are source-transition and sink-transition w.r.t all the edges in 
$M$ as well. For example, consider Figure~\ref{fig:maximal-cycle}. Let's call the green cycle $y_g$, the red cycle $y_r$, and the blue cycle $y_b$.
The $y_g$ has two source-transitions $w_{1,1}, w_{0,3}$, and it has two sink-transitions $w_{0,1}, w_{1,3}$. Since there are no edges going out of $w_{1,3}, w_{0,1}$, and no edges going into $w_{0,3}, w_{1,1}$, $y_g$ is a maximal cycle. Cycles $y_r, y_b$ are maximal cycles as well. Therefore, $\{y_g\cup y_b\cup y_r\}$ gives a maximal cycle decomposition for $M$. However, if we consider embedding of $M$ w.r.t gender ($M_2$), then $y_r$ is not a maximal cycle anymore, and $\{y_g, y_r \cup y_b\}$ gives a maximal cycle decomposition w.r.t $M_2$ and $M_1$(embedding w.r.t countries).
A cycle is called all-maximal cycle if it is maximal w.r.t all the matrix representations $M_{1}, \cdots, M_{|\mathcal{F}|}$. In this example, $\{y_g, y_r \cup y_b\}$ gives an all-maximal cycle decomposition.

\begin{figure}
    \centering
    \includegraphics[width=0.35\textwidth]{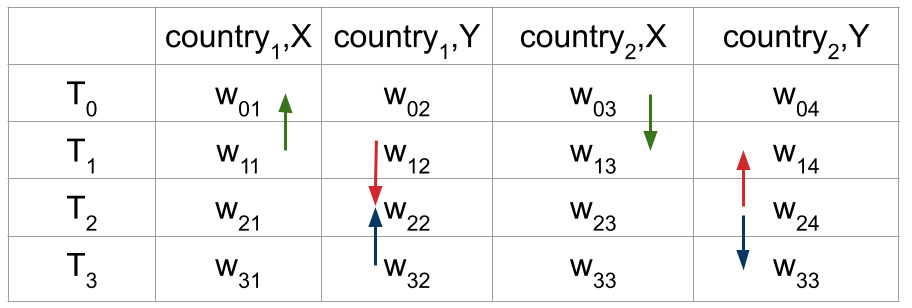}
    \caption{Maximal cycle decomposition}
    \label{fig:maximal-cycle}
\end{figure}
\begin{lemma}
The set of all the edges of $P\oplus P^*$ can be decomposed into a set of all-maximal cycles.
\end{lemma}
\begin{proof}
Consider an arbitrary decomposition of the edges of $P\oplus P^*$ in the matrix representation into a set of cycles $\{y_1,\cdots,y_{\ell}\}$.
If there exists a cycle in $P\oplus P^*$ without any source-transitions and sink-transitions, it means the gain of this cycle is zero and it could be discarded. If there exists any cycle $y_p$ which is not all-maximal, then there exists another cycle $y_q$ which makes $y_p$ not to be maximal w.r.t some features. For example in Figure~\ref{fig:maximal-cycle}, $y_r$ is not maximal because of $y_b$. In this case, union $y_p$ and $y_q$, and make $y_p\cup y_q$ a single cycle in the decomposition. At the end, all the edges in $P\oplus P^*$ will be decomposed into a set of all-maximal cycles. Let's call the set of all-maximal cycles $\{y'_1,\cdots, y'_{\ell'}\}$.
\end{proof}
\begin{theorem}
\label{proof-correctness}
Algorithm~\ref{alg-neg-cycle-detection} finds the global optimum for the diverse $b$-matching problem.
\end{theorem}
\begin{proof}
Let $f(P)$ show the value of the objective function for the assignment $P$. $f(P^*)-f(P)<0$ therefore:
\scriptsize{\[{f(P^*)-f(P) = gain(y'_{1,1}) + gain(y'_{2,2})+\cdots
+gain(y'_{\ell', \ell'}) < 0}\]
}
\normalsize{}
Where $y'_k$ ($1\leq k \leq \ell'$) is the $k^{th}$ cycle in the maximal cycle decomposition, and $y'_{k,k}$ is applying the local exchange of the cycle $y'_{k}$ at step $k$. The initial step is the assignment $P$. Since $f(P^*)-f(P)<0$, there must be a maximal cycle $y'_g$ such that $gain(y'_{g,g})<0$. We wish to show $gain(y'_{g,1})<0$, which implies starting from the initial assignment $P$, a local exchange can be done with a negative gain, and $P$ is not a local optimum which is a contradiction.

Let $D(y'_{g,g}), U(y'_{g,g})$ denote respectively the change in the diversity, and the change in the utility when applying a local exchange $y'_{g}$ in step $g$.
Let $D_{f_k}(y'_{g,g})$ denote the change in the diversity w.r.t the feature $f_k$, when applying $y'_{g,g}$. Therefore:
\[gain(y'_{g,g}) = \sum_{k\in |\mathcal{F}|} D_{f_k}(y'_{g,g}) + U(y'_{g,g})\]
Lemma~\ref{upper-bound-diversity-one-feature} shows if $D_{f_k}(y'_{g,g}) < 0$, then $D_{f_k}(y'_{g,1}) < 0$.
As a result, $D(y'_{g,g}) < 0$ implies $D(y'_{g,1}) < 0$. It is easy to see that $U(y'_{g,g}) = U(y'_{g,1})$. Therefore, $gain(y'_{g,g})<0$ implies $gain(y'_{g,1}) < 0$, and the proof is complete.
\end{proof}
\begin{lemma}\label{upper-bound-diversity-one-feature} If $D_{f_k}(y'_{g,g}) < 0$, then $D_{f_k}(y'_{g,1}) < 0$.
\end{lemma}
\begin{proof}
Consider $y'_{g,g}$ embedded into $M_{k}$. There are four types of vertices in $y'_{g,g}$:
\begin{itemize}
   \item Vertices in the form of $w_{0,j}$  where $1\leq j\leq |\mathcal{V}|$. These vertices have contribution zero to both $D_{f_k}(y'_{g,g})$ and $D_{f_k}(y'_{g,1})$.
   \item Vertices that are not sink-transition or source-transition, i.e. $w_{2,2}$ in Figure~\ref{fig:matrix-exchange}, w.r.t $M_k$. It could be seen that contribution of these nodes to both $D_{f_k}(y'_{g,g})$ and $D_{f_k}(y'_{g,1})$ is zero.
   \item Sink-transitions: Consider an arbitrary sink-transition $v$ in $y'_{g,g}$. Assume the value of this node at the beginning of step $g$ is $v_g$. The contribution of $v$ to $D_{f_k}(y'_{g,g})$ is $\lambda_k\big((v_g+1)^2 - v_g^2\big)>0$. Since $v$ is a sink-transition, $v_g\geq v_1$. Therefore, $\lambda_k\big((v_g+1)^2 - v_g^2\big)\geq\lambda_k\big( (v_1+1)^2 - v_1^2\big)$.
   \item Source-transitions: Consider an arbitrary source-transition $v$ in $y'_{g,g}$. 
   The contribution of $v$ to $D_{f_k}(y'_{g,g})$ is $\lambda_k\big((v_g-1)^2 - v_g^2\big)$. Since $v$ is a source-transition $v_1\geq v_g$, and therefore $\lambda_k\big((v_g-1)^2-v_g^2\big) \geq \lambda_k\big((v_1-1)^2 - v_1^2\big)$.
\end{itemize}
At the end, contribution of all the vertices to $D_{f_k}(y'_{g,1})$ is upper bounded by their contribution to $D_{f_k}(y'_{g,g})$. Therefore if $D_{f_k}(y'_{g,g})<0$, then $D_{f_k}(y'_{g,1})<0$.
\end{proof}
\begin{theorem}
\label{thm:running-time}
The running time of the algorithm is $\mathcal{O}((\lambda_{\max}\cdot|\mathcal{F}|\cdot n^2+\lambda_0U)\cdot |\mathcal{V}|^2\cdot t^2(|\mathcal{V}|+t))$, where $U$ is the maximum cost of an initial feasible b-matching and $\lambda_{max} = \max\{\lambda_1,\cdots, \lambda_{|\mathcal{F}|}\}$.
\end{theorem}
In order to prove this theorem, first we show the following lemmas hold.
\begin{lemma}
\label{lem:num-iter}
The number of iterations of our algorithm is at most $(\lambda_{max}\cdot|\mathcal{F}|\cdot n^2 + \lambda_0 U)$.
\end{lemma}
\begin{proof}
The initial state of the algorithm is a feasible b-matching with cost at most $U$. Diversity of any matching is at most $\lambda_{max}\cdot|\mathcal{F}|\cdot n^2$. At each iteration, we find a negative weight cycle and since all the weights are integers, its weight can be at most $-1$. Therefore, the objective function decreases by at least $1$ at each step, and since the value of the objective function is always positive, the number of iterations is at most $(\lambda_{max}\cdot|\mathcal{F}|\cdot n^2 + \lambda_0 U)$.
\end{proof}
\begin{lemma}
\label{lem:complexity-bellman-ford}
The complexity of each iteration of the algorithm is $\mathcal{O}(|\mathcal{V}|^2\cdot t^2(|\mathcal{V}|+t))$.
\end{lemma}
\begin{proof}
At each iteration, we use a negative cycle detection algorithm with running time $\mathcal{O}(|V|\cdot |E|)$ (where $|V|$ is the number of nodes in the auxiliary graph and $|E|$ is the number of edges). The number of nodes in the graph is $2|\mathcal{V}|\cdot (t+1)$, since there are $t+1$ switches in the graph and each switch has exactly $2|\mathcal{V}|$ ports and each port is a node in the graph. The number of edges incident on each port is $|\mathcal{V}|+t$. Therefore, the total number of edges is $\mathcal{O}(|\mathcal{V}|\cdot t(|\mathcal{V}|+t))$. Hence, the complexity of each iteration is $\mathcal{O}(|\mathcal{V}|^2\cdot t^2(|\mathcal{V}|+t))$.
\end{proof}
Combining Lemma~\ref{lem:num-iter} with Lemma~\ref{lem:complexity-bellman-ford}, and considering $\mathcal{O}(|\mathcal{V}|+t)$ time complexity for finding an initial feasible solution, yields Theorem~\ref{thm:running-time}.

%% file: general_weights.tex
\section{Diverse Weighted Bipartite $b$-Matching}\label{sec:general-weights}
In this section, we extend our algorithm to solve the case where the cost of assigning workers from the \emph{same} country to a team can be different. First, in each switch we put input and output ports for each worker. Inside each switch, there is a complete bipartite graph from input ports to the output ports. Consider an edge between an input port to an output port corresponding to workers $x_i$ and $x_j$. Per each feature $f_k$ where $x_i, x_j$ have the same values for $f_k$, $-2\lambda_k$ is added to the weight of the edge between $x_i, x_j$.

 Consider an edge from output port $x_{k}^{i_1}$ of switch $T_{i_1}$ to input port $x_{k}^{i_2}$ of switch $T_{i_2}$, where $x_{k}\in V_j$. The weight of this edge is equal to the change in the objective function by moving one worker from $V_j$ out of $T_{i_1}$, and adding that worker to $T_{i_2}$.
The proof of the following theorem is similar to Thm.~\ref{thm:running-time}.
\begin{theorem}
The running time of the algorithm for general weights is $\mathcal{O}((\lambda_{max}\cdot |\mathcal{F}|\cdot n^2+\lambda_0U)\cdot n^2\cdot t^2(n+t))$, where $U$ is the maximum cost of any feasible $b$-matching.
\end{theorem}

%% file: experiments.tex
\section{Experimental Validation \& Discussion}\label{sec:experiments}
To demonstrate the efficacy of the proposed method, we apply it to a dataset of reviewer paper matching. First, we find the optimal solution for multi-feature reviewer paper matching and compare it to the single feature diverse matching method. We also provide the MIQP formulation of the same problem based on literature and show how our algorithm is faster to the Gurobi based MIQP solver.

For the reviewer assignment problem, where each reviewer has multiple features, we want to match each paper with reviewers who are not only from different expertise areas (clusters), but also belong to different genders. We use the multi-aspect review assignment evaluation dataset~\cite{karimzadehgan2009constrained}, a benchmark dataset from UIUC. It contains $73$ papers accepted by SIGIR $2007$, and $189$ prospective reviewers who had published in the main information retrieval conferences. The dataset provides $25$ major topics and for each paper in the set, an expert provided $25$-dimensional label on that paper based on a set of defined topics. Similarly for the $189$ reviewers, a $25$-dimensional expertise representation is provided. 

To compare our method (Algorithm~\ref{alg-neg-cycle-detection}) with a baseline, we formulate a multi-feature MIQP variant of our problem, which is an extension of the single-feature formulation provided in~\cite{Ahmed:2017:DWB:3171642.3171649} and is given by:
\begin{align*}
&\min\lambda_0\sum_{i=1}^t \sum_{j=1}^{|\mathcal{F}_1|} u_{i,j}\cdot c_{i,1,j}+\sum_{k=1}^{|\mathcal{F}|} \lambda_k \sum_{i=1}^t\sum_{k'=1}^{|\mathcal{F}_k|}c_{i,k,k'}^2\\
& \sum_{k=1}^{|\mathcal{F}|}\sum_{k'=1}^{|\mathcal{F}_k|} c_{i,k,k'} = d_i, \forall 1\leq i\leq t\\
&\sum_{i=0}^t c_{i,k,k'} = |\mathcal{F}_{k,k'}|, 1\leq k\leq |\mathcal{F}|,1\leq k' \leq |\mathcal{F}_{k}|
\end{align*}
To set up the graph for our method, we first cluster the reviewers into $5$ clusters based on their topic vectors using spectral clustering. To calculate the relevance of each cluster for any paper, we take the average cosine similarity of label vectors of reviewers in that cluster and the paper. We set the constraints such that each paper matches with exactly $4$ reviewers, and no reviewer is allocated more than $1$ paper. To increase dataset size, we double the number of reviewers by creating a copy of each reviewer. As the original dataset lacks gender information, we added a new feature to each reviewer in this dataset by randomly adding one of two gender labels (Male or Female) to each reviewer. We set $\lambda_0 = \lambda_1 = \lambda_2 = 1 $ for our experiments.


 



We run the negative cycle detection algorithm, and the MIQP solver using Gurobi to find the optimum solution. 
On converging to the optimal solution, we find that all $73$ papers receive two male reviewers and two female reviewers, which shows that the method was capable of balancing gender diversity. All papers receive reviewers from four different clusters too. If we only optimize for cluster diversity, it is possible that the gender ratio for individual paper gets skewed. When we run the same model with $\lambda_{g}=0$ (no weight to gender diversity), we find that out of $73$ papers, $12$ papers receive all four reviewers of the same gender and $41$ papers receive three reviewers of the same gender. Hence, only $27.3\%$ teams of reviewers are gender balanced. However, one should note that when we do not keep gender as an objective, the resultant allocation is random and different skewness can be observed in different runs based on the initial solution. 

Finally, we compare the timing performance of our algorithm with MIQP by changing the number of papers that need to be reviewed on a Dell XPS 13 laptop with i7 processor. For MIQP, we set a maximum run time of four hours (14400 seconds) for Gurobi solver, at which we report the current best MIQP solution.
Table~\ref{tab:res} shows that for all cases with the number of papers greater than $13$, MIQP does not converge within four hours, while our method finds the optimum solution in lesser time. Interestingly, MIQP current solutions are found to be the same as the optimum solution found by our method, which shows that for this application, MIQP was able to search the solution but it was not able to prove that the solution is optimum. In contrast, our method finds the solution faster as well as guarantees that it is optimum. 

\begin{table}[!htbp]
\scriptsize
\centering
\begin{tabular}{rrrr}
\toprule
\# Papers & \# Reviewers & MIQP Time (s) & Our Method Time (s)\\
\midrule
03 & 378 & 24.68 & 0.18\\
13 & 378 &  3979.90 & 14.84\\
23 & 378 &  14400.00 & 122.96\\
33 & 378 &  14400.00 & 400.56\\
43 & 378 &  14400.00 & 825.95\\
53 & 378 &  14400.00 & 2837.15\\
63 & 378 &  14400.00 & 5453.58\\
73 & 378 &  14400.00 & 11040.55\\
\bottomrule\\
\end{tabular}
\caption{Comparison of MIQP and our method for UIUC reviewer dataset with each paper needing 4 reviewers.}
\label{tab:res}
\end{table}

%% file: conclusion.tex
\section{Conclusion \& Future Research}\label{sec:conclusions}

In this paper, we proposed the first pseudo-polynomial time algorithms for multi-feature diverse weighted bipartite $b$-matching---a problem that we also showed is NP-hard. We propose an algorithm that not only guarantees an optimal solution but also converges faster than a proposed approach using a black-box industrial MIQP solver. We demonstrated our results on a dataset for paper reviewer matching. Future work could explore the extension of this method to online diverse matching~\cite{Dickerson19:Balancing}, where vertices arrive sequentially and must match immediately; this has direct application in advertising, where one could balance notions of reach, frequency, and immediate monetary return.  Exploring connections to fairness in machine learning~\cite{Grgic-Hlava18:Beyond} and hiring~\cite{Schumann19:Diverse} by way of diversity are also of immediate interest.

%% file: arxiv_submission_013120.bbl
\begin{thebibliography}{}

\bibitem[\protect\citeauthoryear{Agrawal \bgroup \em et al.\egroup
  }{2018}]{Agrawal18:Proportional}
Shipra Agrawal, Morteza Zadimoghaddam, and Vahab Mirrokni.
\newblock Proportional allocation: Simple, distributed, and diverse matching
  with high entropy.
\newblock In {\em International Conference on Machine Learning (ICML)}, pages
  99--108, 2018.

\bibitem[\protect\citeauthoryear{Ahmed and Fuge}{2018}]{ahmed2018ranking}
Faez Ahmed and Mark Fuge.
\newblock Ranking ideas for diversity and quality.
\newblock {\em Journal of Mechanical Design}, 140(1):011101, 2018.

\bibitem[\protect\citeauthoryear{Ahmed \bgroup \em et al.\egroup
  }{2017}]{Ahmed:2017:DWB:3171642.3171649}
Faez Ahmed, John~P. Dickerson, and Mark Fuge.
\newblock Diverse weighted bipartite b-matching.
\newblock In {\em Proceedings of the 26th International Joint Conference on
  Artificial Intelligence}, IJCAI'17, pages 35--41. AAAI Press, 2017.

\bibitem[\protect\citeauthoryear{Basu~Roy \bgroup \em et al.\egroup
  }{2015}]{basu2015task}
Senjuti Basu~Roy, Ioanna Lykourentzou, Saravanan Thirumuruganathan, Sihem
  Amer-Yahia, and Gautam Das.
\newblock Task assignment optimization in knowledge-intensive crowdsourcing.
\newblock {\em The VLDB Journal—The International Journal on Very Large Data
  Bases}, 24(4):467--491, 2015.

\bibitem[\protect\citeauthoryear{Benabbou \bgroup \em et al.\egroup
  }{2018}]{benabbou2018diversity}
Nawal Benabbou, Mithun Chakraborty, Xuan-Vinh Ho, Jakub Sliwinski, and Yair
  Zick.
\newblock Diversity constraints in public housing allocation.
\newblock In {\em Proceedings of the 17th International Conference on
  Autonomous Agents and MultiAgent Systems}, pages 973--981. International
  Foundation for Autonomous Agents and Multiagent Systems, 2018.

\bibitem[\protect\citeauthoryear{Bertsimas \bgroup \em et al.\egroup
  }{2019}]{Bertsimas19:Balancing}
Dimitris Bertsimas, Theodore Papalexopoulos, Nikolaos Trichakis, Yuchen Wang,
  Ryutaro Hirose, and Parsia~A Vagefi.
\newblock Balancing efficiency and fairness in liver transplant access:
  tradeoff curves for the assessment of organ distribution policies.
\newblock {\em Transplantation}, 2019.

\bibitem[\protect\citeauthoryear{Carbonell and
  Goldstein}{1998}]{carbonell1998use}
Jaime Carbonell and Jade Goldstein.
\newblock The use of {MMR}, diversity-based reranking for reordering documents
  and producing summaries.
\newblock In {\em Proceedings of the 21st Annual International ACM SIGIR
  Conference on Research and Development in Information Retrieval}, pages
  335--336. ACM, 1998.

\bibitem[\protect\citeauthoryear{Charlin and Zemel}{2013}]{charlin2013toronto}
Laurent Charlin and Richard Zemel.
\newblock The toronto paper matching system: an automated paper-reviewer
  assignment system.
\newblock 2013.

\bibitem[\protect\citeauthoryear{Cherkassky \bgroup \em et al.\egroup
  }{1993}]{Cherkassky93shortestpaths}
Boris Cherkassky, Andrew~V. Goldberg, and Tomasz Radzik.
\newblock Shortest paths algorithms: Theory and experimental evaluation.
\newblock {\em Mathematical Programming}, 73:129--174, 1993.

\bibitem[\protect\citeauthoryear{Dickerson and
  Sandholm}{2015}]{Dickerson:2015:FCH:2887007.2887094}
John~P. Dickerson and Tuomas Sandholm.
\newblock Futurematch: Combining human value judgments and machine learning to
  match in dynamic environments.
\newblock In {\em Proceedings of the Twenty-Ninth AAAI Conference on Artificial
  Intelligence}, AAAI'15, pages 622--628. AAAI Press, 2015.

\bibitem[\protect\citeauthoryear{Dickerson \bgroup \em et al.\egroup
  }{2019}]{Dickerson19:Balancing}
John~P. Dickerson, Karthik~Abinav Sankararaman, Aravind Srinivasan, and Pan Xu.
\newblock Balancing relevance and diversity in online bipartite matching via
  submodularity.
\newblock In {\em AAAI Conference on Artificial Intelligence (AAAI)}, 2019.

\bibitem[\protect\citeauthoryear{Drummond \bgroup \em et al.\egroup
  }{2015}]{drummond2015sat}
Joanna Drummond, Andrew Perrault, and Fahiem Bacchus.
\newblock Sat is an effective and complete method for solving stable matching
  problems with couples.
\newblock In {\em IJCAI}, pages 518--525, 2015.

\bibitem[\protect\citeauthoryear{Goldberg and
  Radzik}{1993}]{Goldberg93aheuristic}
Andrew~V. Goldberg and Tomasz Radzik.
\newblock A heuristic improvement of the bellman-ford algorithm, 1993.

\bibitem[\protect\citeauthoryear{G{\"o}lz and
  Procaccia}{2019}]{golz2018migration}
Paul G{\"o}lz and Ariel~D Procaccia.
\newblock Migration as submodular optimization.
\newblock In {\em AAAI Conference on Artificial Intelligence (AAAI)}, 2019.

\bibitem[\protect\citeauthoryear{Grgi{\'c}-Hla{\v{c}}a \bgroup \em et
  al.\egroup }{2018}]{Grgic-Hlava18:Beyond}
Nina Grgi{\'c}-Hla{\v{c}}a, Muhammad~Bilal Zafar, Krishna~P Gummadi, and Adrian
  Weller.
\newblock Beyond distributive fairness in algorithmic decision making: Feature
  selection for procedurally fair learning.
\newblock In {\em AAAI Conference on Artificial Intelligence (AAAI)}, 2018.

\bibitem[\protect\citeauthoryear{Hirschman}{1964}]{hirschman1964paternity}
Albert~O Hirschman.
\newblock The paternity of an index.
\newblock {\em The American economic review}, 54(5):761--762, 1964.

\bibitem[\protect\citeauthoryear{Hunt \bgroup \em et al.\egroup
  }{2015}]{hunt2015diversity}
Vivian Hunt, Dennis Layton, and Sara Prince.
\newblock Diversity matters.
\newblock {\em McKinsey \& Company}, 1:15--29, 2015.

\bibitem[\protect\citeauthoryear{Karimzadehgan and
  Zhai}{2009}]{karimzadehgan2009constrained}
Maryam Karimzadehgan and ChengXiang Zhai.
\newblock Constrained multi-aspect expertise matching for committee review
  assignment.
\newblock In {\em ACM Conference on Information and Knowledge Management
  (CIKM)}, pages 1697--1700, 2009.

\bibitem[\protect\citeauthoryear{Kobren \bgroup \em et al.\egroup
  }{2019}]{Kobren:2019:PML:3292500.3330899}
Ari Kobren, Barna Saha, and Andrew McCallum.
\newblock Paper matching with local fairness constraints.
\newblock In {\em Proceedings of the 25th ACM SIGKDD International Conference
  on Knowledge Discovery \& Data Mining}, KDD '19, pages 1247--1257, New York,
  NY, USA, 2019. ACM.

\bibitem[\protect\citeauthoryear{Kulesza \bgroup \em et al.\egroup
  }{2012}]{kulesza2012determinantal}
Alex Kulesza, Ben Taskar, et~al.
\newblock Determinantal point processes for machine learning.
\newblock {\em Foundations and Trends{\textregistered} in Machine Learning},
  5(2--3):123--286, 2012.

\bibitem[\protect\citeauthoryear{Kurata \bgroup \em et al.\egroup
  }{2017}]{kurata2017controlled}
Ryoji Kurata, Naoto Hamada, Atsushi Iwasaki, and Makoto Yokoo.
\newblock Controlled school choice with soft bounds and overlapping types.
\newblock {\em Journal of Artificial Intelligence Research}, 58:153--184, 2017.

\bibitem[\protect\citeauthoryear{Lian \bgroup \em et al.\egroup
  }{2018}]{lian2018conference}
Jing~Wu Lian, Nicholas Mattei, Renee Noble, and Toby Walsh.
\newblock The conference paper assignment problem: Using order weighted
  averages to assign indivisible goods.
\newblock In {\em Thirty-Second AAAI Conference on Artificial Intelligence},
  2018.

\bibitem[\protect\citeauthoryear{Lin and Bilmes}{2012}]{lin2012learning}
Hui Lin and Jeff~A Bilmes.
\newblock Learning mixtures of submodular shells with application to document
  summarization.
\newblock {\em arXiv preprint arXiv:1210.4871}, 2012.

\bibitem[\protect\citeauthoryear{Liu \bgroup \em et al.\egroup
  }{2014}]{Liu:2014:RMP:2645710.2645749}
Xiang Liu, Torsten Suel, and Nasir Memon.
\newblock A robust model for paper reviewer assignment.
\newblock In {\em Proceedings of the 8th ACM Conference on Recommender
  Systems}, RecSys '14, pages 25--32, New York, NY, USA, 2014. ACM.

\bibitem[\protect\citeauthoryear{Minoux}{1986}]{Minoux1986}
M.~Minoux.
\newblock {\em Solving integer minimum cost flows with separable convex cost
  objective polynomially}, pages 237--239.
\newblock Springer Berlin Heidelberg, Berlin, Heidelberg, 1986.

\bibitem[\protect\citeauthoryear{{\O}stergaard \bgroup \em et al.\egroup
  }{2011}]{ostergaard2011does}
Christian~R {\O}stergaard, Bram Timmermans, and Kari Kristinsson.
\newblock Does a different view create something new? the effect of employee
  diversity on innovation.
\newblock {\em Research Policy}, 40(3):500--509, 2011.

\bibitem[\protect\citeauthoryear{Ross \bgroup \em et al.\egroup
  }{2010}]{ross2010crowdworkers}
Joel Ross, Lilly Irani, M~Silberman, Andrew Zaldivar, and Bill Tomlinson.
\newblock Who are the crowdworkers?: shifting demographics in mechanical turk.
\newblock In {\em CHI'10 extended abstracts on Human factors in computing
  systems}, pages 2863--2872. ACM, 2010.

\bibitem[\protect\citeauthoryear{Schumann \bgroup \em et al.\egroup
  }{2019}]{Schumann19:Diverse}
Candice Schumann, Samsara~N. Counts, Jeffrey Foster, and John~P. Dickerson.
\newblock The diverse cohort selection problem.
\newblock In {\em International Conference on Autonomous Agents and Multi-Agent
  Systems (AAMAS)}, 2019.

\end{thebibliography}
